\documentclass{llncs}
\usepackage{amsmath}
\usepackage{amssymb}
\usepackage{times}
\usepackage{indentfirst}
\usepackage{graphicx}
\makeatletter
\newif\if@borderstar
\def\bordermatrix{\@ifnextchar*{%
\@borderstartrue\@bordermatrix@i}{\@borderstarfalse\@bordermatrix@i*}%
}
\def\@bordermatrix@i*{\@ifnextchar[{\@bordermatrix@ii}{\@bordermatrix@ii[()]}}
\def\@bordermatrix@ii[#1]#2{%
\begingroup
\m@th\@tempdima8.75\p@\setbox\z@\vbox{%
\def\cr{\crcr\noalign{\kern 2\p@\global\let\cr\endline }}%
\ialign {$##$\hfil\kern 2\p@\kern\@tempdima &\thinspace %
\hfil $##$\hfil &&\quad\hfil $##$\hfil\crcr\omit\strut %
\hfil\crcr\noalign{\kern -\baselineskip}#2\crcr\omit %
\strut\cr}}%
\setbox\tw@\vbox{\unvcopy\z@\global\setbox\@ne\lastbox}%
\setbox\tw@\hbox{\unhbox\@ne\unskip\global\setbox\@ne\lastbox}%
\setbox\tw@\hbox{%
$\kern\wd\@ne\kern -\@tempdima\left\@firstoftwo#1%
\if@borderstar\kern2pt\else\kern -\wd\@ne\fi%
\global\setbox\@ne\vbox{\box\@ne\if@borderstar\else\kern 2\p@\fi}%
\vcenter{\if@borderstar\else\kern -\ht\@ne\fi%
\unvbox\z@\kern-\if@borderstar2\fi\baselineskip}%
\if@borderstar\kern-2\@tempdima\kern2\p@\else\,\fi\right\@secondoftwo#1$%
}\null \;\vbox{\kern\ht\@ne\box\tw@}%
\endgroup
}
\makeatother

\begin{document}

\title{Topological characterizations to three types of covering approximation operators
}         

\author{Aiping Huang, William Zhu~\thanks{Corresponding author.
E-mail: williamfengzhu@gmail.com (William Zhu)}}
\institute{
Lab of Granular Computing,\\
Zhangzhou Normal University, Zhangzhou 363000, China}



\date{\today}          
\maketitle
\begin{abstract}
Covering-based rough set theory is a useful tool to deal with inexact, uncertain or vague knowledge in information systems.
Topology, one of the most important subjects in mathematics, provides mathematical tools and interesting topics in studying information systems
and rough sets.
In this paper, we present the topological characterizations to three types of covering approximation operators.
First, we study the properties of topology induced by the sixth type of covering lower approximation operator.
Second, some topological characterizations to the covering lower approximation operator to be an interior operator are established.
We find that the topologies induced by this operator and by the sixth type of covering lower approximation operator are the same.
Third, we study the conditions which make the first type of covering upper approximation operator be a closure operator, and find that
the topology induced by the operator is the same as the topology induced by the fifth type of covering upper approximation operator.
Forth, the conditions of the second type of covering upper approximation operator to be a closure operator and the properties of topology
induced by it are established.
Finally, these three topologies space are compared.
In a word, topology provides a useful method to study the covering-based rough sets.


\end{abstract}

\section*{Introduction}

Rough set theory \cite{Pawlak82Rough} was proposed by Pawlak to deal with granularity in information systems.
It is based on equivalence relation.
However, the applications of equivalent relation are quite limited, hence classical rough set theory has been extended to tolerance relation~\cite{SkowronStepaniuk96tolerance}, similarity relation~\cite{SlowinskiVanderpooten00AGeneralized} and even arbitrary
binary relation~\cite{LiuZhu08TheAlgebraic,Yao98OnGeneralizing,Yao98Relational,Yao98Constructive,ZhuWang06ANew}.
Through extending a partition to a covering, rough set theory is generalized to covering-based rough sets
~\cite{QinGaoPei07OnCovering,WangZhuZhu10Structure,WangZhu12Quantitative,ZhuWang02Some}.
Because of its high efficiency in many complicated problems such as attribute reduction and rule learning in incomplete
information/decision~\cite{QianLiangLiWangMa10Approximation}, covering-based rough set theory has been attracting increasing
research interest~\cite{WangYangYangWu12Relationships,YaoYao12Covering}.

Topology, one of the most important subjects in mathematics, provides mathematical tools and interesting topics in studying information systems
and rough sets~\cite{Kondo05OnTheStructure,QinPei05OntheTopological,QinYangPei08Generalized}.
This connects rough set theory with topological theory which have deep theoretical and practical significance beyond doubt.
Indeed, Polkowski in \cite{Lech02Rough} pointed: topological spaces of rough set theory were recognized early in the framework of topology of
partitions. Skowron et al.~\cite{SkowronStepaniuk96tolerance} generalized the classical approximation spaces to tolerance approximation
spaces, and discussed the problems of attribute reduction in these spaces.
In addition, connections between fuzzy rough set theory and fuzzy topology were also investigated in \cite{LiZhang08Rough,QinPei05OntheTopological}.
In a word, topology provides an interesting and natural research topic in rough set theory.

However, there are some problems still remain to be solved.
For example, the topological characterizations to covering approximation operators to be closure or interior operators.
Gei et al. has discussed this topic in \cite{GeBaiYun12Topological}.
In contrast, we establish other characterizations to covering approximation operators to be closure or interior operators.
Moreover, the properties of topologies induced by them are studied in this paper.
The sixth type of covering lower approximation operator is an interior operator without any limiting conditions, then it can induce a topology.
The properties of the topology are discussed in the first part of this paper.
In the second part, some topological characterizations to covering lower approximation operator to be an interior operator are established.
It is interesting to find that the topologies induced by this operator and by the sixth type of covering lower approximation operator are the same.
In the third part, we study the conditions which make the first type of covering upper approximation operator be a closure operator, and find that
the topology induced by the operator is different from the one induced by lower approximation operator.
However, it is the same as the topology induced by the fifth type of covering upper approximation operator.
The conditions of the second type of covering approximation operator are presented in the forth part of this paper.
We find that the topology induced by this operator is a preudo-discrete topology.
Finally, these three topologies space are compared.
In fact, we discuss total five types of covering approximation operators, however, there are two pair of covering approximation operators induce
the same topological space, namely, the first and the fifth, and the lower and the sixth type of covering lower approximation operator.
Hence, there are three types of topological structures in this paper.
In a word, topology provides a useful method to study the covering-based rough sets.

The reminder of this paper is organized as follows.
In Section \ref{S:Basic definitions}, we present the fundamental concepts and properties of approximation operators in covering-based rough sets
and topologys.
Section \ref{Thepropertiesoftopologyinducedby} studies the topological characterizations to the sixth type of covering upper approximation operator.
In section \ref{Topologicalcharacterizationtothecoveringlowerapproximationoperator}, we study some conditions which make the lower covering
approximation operator be an interior operator.
Section \ref{Theconditionofbeingatopologyandthetopologyinducedby} studies the conditions which make the first type of covering upper approximation
operator be a closure operator.
In section \ref{Topologicalcharacterizationtothesecondtypeofcoveringapproximationoperators}, the condition of the second type of
covering upper approximation operator to be a closure operator and the properties of topology induced by it are presented.
Section \ref{Relationshipsamongthesethreetopologicalspaces}, we compare this three types of topological spaces.
This paper concludes in Section \ref{Conslusion}.
\section{Basic definitions}
\label{S:Basic definitions}

In this section, we introduce the fundamental ideas about Pawlak's Rough sets, coverings, and the existing five types of covering-based rough sets.

\subsection{Fundamentals of Pawlak's Rough sets}
\label{FundamentalsofPawlak'sRoughsets}

Let $U$ be a finite set and $R$ be an equivalence relation on $U$.
$R$ will generate a partition $U/R=\{Y_{1},Y_{2},\cdots,Y_{m}\}$ on $U$, where $Y_{1},Y_{2},\cdots,Y_{m}$ are the equivalence classes
generated by $R$.
$\forall X\subseteq U$, the lower and upper approximations of $X$, are, respectively, defined as follows:\\
$R_{\ast}(X)=\bigcup\{Y_{i}\in U/R:Y_{i}\subseteq X\}$,\\
$R^{\ast}(X)=\bigcup\{Y_{i}\in U/R:Y_{i}\bigcap X\neq \emptyset\}$.

\begin{proposition}\label{TheproperitiesofPawlakRoughset}
Let $\emptyset$ be the empty set and $-X$ the complement of $X$ in $U$.
Pawlak's rough sets have the following properties:\\
(1L) $R_{\ast}(U)=U$~~~~~~~~~~~~~~~~~~~~~~~~~~~~~~~~~~~~~~(1H) $R^{\ast}(U)=U$\\
(2L) $R_{\ast}(\emptyset)=\emptyset$~~~~~~~~~~~~~~~~~~~~~~~~~~~~~~~~~~~~~~~~(2H) $R^{\ast}(\emptyset)=\emptyset$\\
(3L) $R_{\ast}(X)\subseteq X$~~~~~~~~~~~~~~~~~~~~~~~~~~~~~~~~~~~~~(3H) $X\subseteq R^{\ast}(X)$\\
(4L) $R_{\ast}(X\bigcap Y)=R_{\ast}(X)\bigcap R_{\ast}(Y)$~~~~~(4H) $R^{\ast}(X\bigcup Y)=R^{\ast}(X)\bigcup R^{\ast}(Y)$\\
(5L) $R_{\ast}(R_{\ast}(X))=R_{\ast}(X)$~~~~~~~~~~~~~~~~~~~~~(5H) $R^{\ast}(R^{\ast})(X)=R^{\ast}(X)$\\
(6LH) $R_{\ast}(-X)=-R^{\ast}(X)$\\
(7L) $X\subseteq Y\Rightarrow R_{\ast}(X)\subseteq R_{\ast}(X)$~~~~~~~~~~~(7H) $X\subseteq Y\Rightarrow R^{\ast}(X)\subseteq R^{\ast}(X)$\\
(8L) $R_{\ast}(-R_{\ast}(X))=-R_{\ast}(X)$~~~~~~~~~~~~~~~(8H) $R^{\ast}(-R^{\ast})(X)=-R^{\ast}(X)$\\
(9L) $\forall K\in U/R, R_{\ast}(K)=K$~~~~~~~~~~~~~~~~~(9H) $\forall K\in U/R, R^{\ast}(K)=K$
\end{proposition}

\subsection{Covering-Based rough sets}
In this section, we present some basic concepts of covering-based rough sets that are used in this paper. $P(U)$ denotes the family of all subsets
of $U$.
$\mathcal{C}$ is a family of nonempty subsets of $U$.
If $\bigcup \mathcal{C}=U$, then $\mathcal{C}$ is called a covering of $U$ and the order pair $(U,\mathcal{C})$ a covering approximation space.
Let $\mathcal{C}$ be a covering of $U$, $x\in U$.
Denote $Md(x)=\{K\in \mathcal{C}:x\in K\bigwedge (\forall S\in \mathcal{C}\bigwedge x\in S\bigwedge S\subseteq K\Rightarrow K=S)\}$,
$I(x)=\bigcup_{x\in K}K$, and $N(x)=\bigcap_{x\in K}K$.
$Md(x)$, $I(x)$, and $N(x)$, which are called the minimal description of $x$, the indiscernible neighborhood of $x$, and the neighborhood of $x$,
respectively, are first proposed in \cite{BonikowskiBryniarskiWybraniecSkardowska98Extensions}, \cite{ZhuWang06Relationships}
and \cite{Zhu07Topological}.
If $x\in U$, $|Md(x)|=1$, then $\mathcal{C}$ is called unary covering.
This concept was first proposed in \cite{ZhuWang06Relationships}.

In this paper, we study the following types of covering approximation operators.

\begin{definition}\label{definitionofloweroperator}
Let $\mathcal{C}$ be a covering of $U$. The operations $CL$, $SL$ and $XL:P(U)\rightarrow P(U)$ are defined as follows, respectively:
$\forall X\in P(U)$,\\
$CL(X)=\bigcup\{K\in \mathcal{C}:K\subseteq X\}$,\\
$SL(X)=\{x\in U:\forall K\in \mathcal{C}(x\in K\Rightarrow K \subseteq X)\}=\{x\in X: I(x)\subseteq X\}$,\\
$XL(X)=\{x\in U:N(x)\subseteq X\}$.\\
And operators $FH,SH,IH,XH:P(U)\rightarrow P(U)$ are defined as follows, respectively: $\forall X\in P(U)$,\\
$FH(X)=CL(X)\bigcup (\bigcup \{\bigcup Md(x):x\in (X-CL(X))\})$,\\
$SH(X)=\bigcup\{K:K\bigcap X\neq \emptyset\}=\bigcup \{I(\{x\}): x\in X\}$,\\
$IH(X)=CL(X)\bigcup \{N(x):x\in X-CL(X)\}=\bigcup_{x\in X}N(x)$,\\
$XH(X)=\{x\in U:N(x)\bigcap X\neq \emptyset\}.$
\end{definition}

The $SH$ and $SL$, and $XL$ and $XH$ mentioned in above definition are dual, respectively.
Corresponding to the properties of Pawlak's rough sets listed in Section \ref{FundamentalsofPawlak'sRoughsets}, we have the following results.

\begin{proposition}\cite{ZhuWang03Reduction}\label{propertiesofthelowerapproximationoperator}
$CL$ has properties (1L), (2L), (3L), (5L), (7L), and (9L) in Proposition \ref{TheproperitiesofPawlakRoughset}.
\end{proposition}

\begin{proposition}\cite{ZhuWang07OnThree}\label{propertiesofthesecondapproximationoperator}
$SL$ has properties (1L), (2L), (3L), (4L) and (7L) in Proposition \ref{TheproperitiesofPawlakRoughset}, 
and $SH$ has the properties (1H), (2H), (3H), (4H) and (7H) in Proposition \ref{TheproperitiesofPawlakRoughset}.
\end{proposition}

\begin{proposition}\cite{Zhu09RelationshipBetween}\label{propertiesofthesixtypeofapproximation}
$XL$ has properties (1L), (2L), (3L), (4L), (5L), (7L) and (9L) in Proposition \ref{TheproperitiesofPawlakRoughset}, and $XH$
has properties (1H), (2H), (3H), (4H), (5H) and (7H) in Proposition \ref{TheproperitiesofPawlakRoughset}.
\end{proposition}

\begin{proposition}\cite{Zhu07Topological}\label{propertiesofthefifthtypeofapproximation}
$IH$ has properties (1H), (2H), (3H), (4H), (5H), (7H) and (9H) in Proposition \ref{TheproperitiesofPawlakRoughset}.
\end{proposition}

\subsection{Some basic concepts of topology}
The following topological concepts are elementary and can be found in \cite{Engelking77General}.

\begin{definition}
A topological space is a pair $(U,\mathcal{T})$ consisting of a set $U$ and family $\mathcal{T}$ of a subset of $U$ satisfying the following
conditions:\\
(O1) $U, \emptyset \in \mathcal{T}$.\\
(O2) $\mathcal{T}$ is closed under arbitrary union.\\
(O3) $\mathcal{T}$ is closed under finite intersection.\\
The pair $(U,\mathcal{T})$ is called a topological space.
\end{definition}

The subsets of $U$ belonging to $\mathcal{T}$ are called open sets in space, and their complement are called closed sets in the space.
A subset $X$ in a topological space $(U,\mathcal{T})$ is a neighborhood of $x\in U$ if $X$ contains an open set to which $x$ belongs.
In a topological space $(U,\mathcal{T})$, a family $\mathcal{B}\subseteq \mathcal{T}$ of sets is called a base for topology $\mathcal{T}$
if for each point $x$ of the space, and each neighborhood $U$ of $x$, there is a member $V$ of $\mathcal{B}$ such that $x\in V\subseteq U$.
A cover of a set $U$ is a collection of sets whose union contains $U$ as a subset.
We say that the topological space $(U, \mathcal{T})$ is compact provided that every open cover of $U$ has a finite subcover,
and is $Lindel$\"{o}$f$ space provided that every open cover has a countable subcover.
The topological space is said to be disconnected if it is the union of two disjoint nonempty open sets; otherwise, it is said to be connected.
We say that $(U, \mathcal{T})$ is locally connected at $x$ if for every open set $A$ containing $x$ there exists a connected, open set $V$ with 
$x\in V\subseteq A$, And it is said to be locally connected if it is locally connected at $x$ for all $x$ in $U$.
The maximal connected subsets (ordered by inclusion) of a nonempty topological space are called the connected components of the space.
A space $U$ is said to be first-countable if each point has a countable neighborhood basis (local base) and it is said to be second-countable 
if its topology has a countable base.
A subset $A$ of a topological space $(U,\mathcal{T})$ is dense in $U$ if for any point $x$ in $U$, any neighborhood of $x$ contains at least 
one point from $A$, and this topological space is called separable if it contains a countable dense subset.

\begin{definition}(interior and closure operators).
Let $(U,\mathcal{T})$ be a topological space, a closure (resp. interior) operator $cl:P(U)\rightarrow P(U)$ (resp. $i:P(U)\rightarrow \mathcal{T}$),
where $P(U)$ is the power set of $U$, if it satisfies the following axioms, then we call it a closure operator (resp. interior) on $U$.
$\forall X,Y\subseteq U$:\\
(I): $cl(X\bigcup Y)=cl(X)\bigcup cl(Y)$ (resp.$i(X\bigcap Y)=i(X)\bigcap i(Y)$),\\
(II): $X\subseteq cl(X)$ (resp.$i(X)\subseteq X$),\\
(III): $cl(\emptyset)=\emptyset$ (resp.$i(U)=U$),\\
(IV): $cl(cl(X))=cl(X)$ (resp.$i(i(X))=i(X)$).
\end{definition}

The closure of $A$ of a topological space $(U,\mathcal{T})$ is the intersection of the family of all closed sets containing $A$, while
the interior of $A$ of a topological space $(U,\mathcal{T})$ is the union of the family of all open sets included in $A$.
It is well known that a closure operator $cl$ on $U$ can induce a topology $\mathcal{T}_{cl}=\{-X:cl(X)=X\}$ such that in the topological space
$(U,\mathcal{T}_{cl})$, $cl(A)$ is just the closure of $A$ for each $A\subseteq U$, the similar statement is also true for an interior operator.
In the following discussion, unless it is mentioned specially, the universe of discourse $U$ is considered infinite. 

\section{Topological characterization to the sixth type of covering approximation operators}
\label{Thepropertiesoftopologyinducedby}
The sixth type of covering-based approximation operator was first defined in~\cite{XuWang05On}.
Xu and Wang introduced this type of covering-based rough set model and studied the relationship between
it and binary relation based on rough set model.
Zhu presented the definition of this type of rough sets based on coverings.
Since the sixth type of covering upper approximation operator satisfies the four conditions of closure operator, it can determine a topology.
Hence we present the topological characterizations to the operator firstly.

\begin{theorem}\label{Thesixthtypeoflowerapproximationsandtopology}
Let $\mathcal{C}$ be a covering of $U$. $\mathcal{T}_{XL}=\{X\subseteq U: XL(X) = X\}$ is a topology on $U$.
\end{theorem}

\begin{proof}
In order to prove this result, we need to prove $\mathcal{T}_{XL}=\{X\subseteq U: XL(X) = X\}$ satisfies the topological atoms.

(O1): According to (1L), (2L) of Proposition \ref{propertiesofthesixtypeofapproximation}, we know $\emptyset, U\in \mathcal{T}_{XL}$.

(O2): If $X, Y\in \mathcal{T}_{XL}$, then $XL(X) = X$, $XL(Y) = Y$. According to (4L) of Proposition \ref{propertiesofthesixtypeofapproximation},
we have $XL(X\bigcap Y)=XL(X)\bigcap XL(Y)=X\bigcap Y$, thus $X\bigcap Y\in \mathcal{T}_{XL}$.

(O3): If $\mathcal{T}_{1}\subseteq \mathcal{T}_{XL}$, then $XL(X)=X$ for all $X\in \mathcal{T}_{1}$.
$\bigcup_{X\in \mathcal{T}_{1}}X=\bigcup_{X\in \mathcal{T}_{1}}XL(X)$ $\subseteq XL(\bigcup_{X\in \mathcal{T}_{1}}X)\subseteq
\bigcup_{X\in \mathcal{T}_{1}}X$.
Hence, $XL(\bigcup_{X\in \mathcal{T}_{1}}X)=\bigcup_{X\in \mathcal{T}_{1}}X$,
that is, $\bigcup_{X\in \mathcal{T}_{1}}X$ $\in \mathcal{T}_{XL}$.

Therefore, $\mathcal{T}_{XL}=\{X\subseteq U:XL(X)=X\}$ is a topology on $U$.
\end{proof}

The proposition below establishes another expression of topology induced by the sixth type of lower approximation operator,
and finds the interior and the closure operators of the topology.

\begin{proposition}\label{theotherexpressionofthesixthtypeoflowerapproxiamtion}
Let $(U,\mathcal{T}_{XL})$ be a topological space. $\mathcal{T}_{XL}=\{X\subseteq U:XL(X)=X\}=\{XL(X):X\subseteq U\}$.
Moreover, $XL$, $XH$ are respectively the interior operator and the closure
operator of $\mathcal{T}_{XL}$.
\end{proposition}

\begin{proof}
Since $\forall X\subseteq U, XL(XL(X))=XL(X)$, $\{XL(X):X\subseteq U\}\subseteq \{X\subseteq U:XL(X)=X\}$.
On the other hand, $\{X\subseteq U:XL(X)=X\}\subseteq \{XL(X):X\subseteq U\}$ is trivial.
Hence, we have $\mathcal{T}_{XL}=\{X\subseteq U:XL(X)=X\}=\{XL(X):X\subseteq U\}$.
Assume $i$ and $cl$ are respectively the interior operator and the closure operator of $\mathcal{T}_{XL}$.
Since $XL(X)$ is open and $XL(X)\subseteq X$, then $XL(X)\subseteq i(X)$.
Since $i(X)=\bigcup\{Y: Y\in \mathcal{T}_{XL}, Y\subseteq X\}=\bigcup\{Y: XL(Y)=Y, XL(Y)\subseteq XL(X)\}\subseteq XL(X)$,
thus $i(X)\subseteq XL(X)$.
Therefore, the $XL$ is the interior approximation operator of $\mathcal{T}_{XL}$.
By the duality of $XL$ and $XH$, $XH$ is the closure operator of $\mathcal{T}_{XL}$.
\end{proof}

The following lemma represents another expression of the sixth type of lower approximation operator. Based on the lemma, we can obtain some
topological properties of this type of operator.

\begin{lemma}\label{Theotherexpressofthesixtypeoflowerapproximation}
Let $\mathcal{C}$ be a covering of $U$. For all $X\subseteq U$,
\begin{center}
   $XL(X)=\{x|N(x)\subseteq X\}=\bigcup\{N(x)|N(x)\subseteq X\}$.
\end{center}

\end{lemma}

\begin{proof}
On one hand, $XL(X)\subseteq \bigcup\{N(x)|N(x)\subseteq X\}$ is obvious.
On the other hand, for all $y\in \bigcup\{N(x)|N(x)\subseteq X\}$, there exists $x\in U$ such that $y\in N(x)$ and $N(x)\subseteq X$.
Since $y\in N(x)$, $N(y)\subseteq N(x)\subseteq X$, that is, $y\in XL(X)$, thus, $\bigcup\{N(x)|N(x)\subseteq X\}\subseteq XL(X)$.
Therefore, we have $XL(X)=\{x|N(x)\subseteq X\}=\bigcup\{N(x)|N(x)\subseteq X\}$.
\end{proof}

\begin{theorem}\label{thepropertiesoftopologyinducedbythesixtypeofapproximation}
Let $(U, \mathcal{T}_{XL})$ be a topological space. If $x\in A$, then\\
(1) $\{N(x):x\in U\}$ is a base of $\mathcal{T}_{XL}$.\\
(2) If $A$ is an open set which contains $x$, then $N(x)\subseteq A$.\\
(3) $\{N(x)\}$ is an open neighborhood base of $x$.\\
(4) $N(x)$ is a compact subset of $(U, \mathcal{T}_{XL})$.\\
(5) Each connected component is an open set.\\
(6) $(U, \mathcal{T}_{XL})$ is a first countable space.\\
(7) $(U, \mathcal{T}_{XL})$ is a locally compact space.
\end{theorem}

\begin{proof}
(1): On one hand, according to Proposition \ref{propertiesofthesixtypeofapproximation}, we have $XL(N(x))\subseteq N(x)$.
For all $y\in N(x)$, $N(y)\subseteq N(x)$. Based on the definition of $XL$, we know $y\in XL(N(x))$.
Hence, $N(x)\subseteq XL(N(x))$, that is, $\{N(x):x\in U\}\in \mathcal{T}_{XL}$.
On the other hand, according to lemma \ref{Theotherexpressofthesixtypeoflowerapproximation}, we know for all
$X\in \mathcal{T}_{XL}$, $X=XL(X)=\bigcup\{N(x)|N(x)\subseteq X\}$.
Therefore, $\{N(X):x\in U\}$ is a base of $\mathcal{T}_{XL}$.

(2): According to (1) and $A$ is an open set which contains $x$, there exists $y\in U$ such that $x\in N(y)\subseteq A$,
thus $N(x)\subseteq A$.

(3): According to (1), $\forall x\in U$, $N(x)$ is an open set.
$\forall A\in \mathcal{U}_{x}$ where $\mathcal{U}_{x}$ denotes the set of neighborhood of $x$, there exists open set $V$ such that
$x\in V\subseteq A$ according to the definition of neighborhood.
From (2), we obtain $N(x)\subseteq V \subseteq A$, thus $\{N(x)\}$ is an open neighborhood base of $x$.

(4): Let $\{A_{\alpha}:\alpha \in J\}$ be an open covering of $N(x)$.
We have $N(x)\subseteq \bigcup\{A_{\alpha}:\alpha \in J\}$,
thus there exists $\alpha \in J$ such that $x\in A_{\alpha}$, according to (2), $N(x)\subseteq A_{\alpha}$.
Hence, $N(x)$ is a compact subset of $(U, \mathcal{T}_{XL})$.

(5): Let $C_{x}$ be a connected component containing $x$.
According to the definition of $XH$, we know for all $y\in N(x)$, $x\in XH(\{y\})$ and $y\in XH(y)$.
From Proposition \ref{theotherexpressionofthesixthtypeoflowerapproxiamtion}, we know $XH$ is the closure operator of $\mathcal{T}_{XL}$.
If $\{y\}$ is connected, then $XH(\{y\})$ is connected.
Based on the definition of connected component and $y\in XH(y)$, $y\in XH(\{y\})\subseteq C_{x}$ which shows that $N(x)\subseteq C_{x}$.
From (1), we obtain $C_{x}$ is an open set.

(6): (6) is straightforward from (3).

(7): If $(U, \mathcal{T}_{XL})$ is a locally compact space, then for each point in the space has a locally base which composed of compact
neighborhood. Hence, we can obtain the result by (3) and (4).
\end{proof}

Separations $T_{i}(i=0,1,2)$ of topological spaces are important topological properties and applied or extended into many branches of mathematics.
Next, Some characterizations of separation are established.

\begin{definition}\cite{Engelking77General}
A topological space $(U,\mathcal{T})$ is called a $T_{0}$ space if for any two different points $x,y\in U$, there exists an open set
$A$ such that $x\in A, y\notin A$ or an open set $B$ such that $x\notin B, y\in B$.
\end{definition}

\begin{proposition}
Let $(U,\mathcal{C})$ be a covering approximation space. The following statements are equivalent.\\
(1) $(U, \mathcal{T}_{XL})$ is a $T_{0}$ space.\\
(2) $\forall x,y\in U$, if $x\neq y$, then $x\notin XH(\{y\})$ or $y \notin XH(\{x\})$.\\
(3) $\forall x,y\in U$, if $x\neq y$, then $XH(\{x\})\neq XH(\{y\})$.
\end{proposition}

\begin{proof}
$(1)\Rightarrow (2)$: If $(U,\mathcal{T}_{XL})$ is a $T_{0}$ space, then for all $x,y\in U$ and $x\neq y$, there exists open set $A$ such
that $x\in A, y\notin A$ (or $x\notin B, y\in B$).
According to (2) of Theorem \ref{thepropertiesoftopologyinducedbythesixtypeofapproximation}, we have $N(x)\subseteq A$ (or $N(y)\subseteq B$).
Since $y\notin A$ (or $x\in B$), $y\notin N(x)$ (or $x\notin N(y)$).
According to the definition of $XH$, we have $x\notin XH(\{y\})$ (or $y\notin XH(\{x\})$).

$(2)\Rightarrow (3)$: According to Proposition \ref{propertiesofthesixtypeofapproximation}, for all $x\in U$, we have $x\in XH(\{x\})$.
From (2), we know for all $y\in U$ and $y\neq x$, $x\notin XH(\{y\})$.
Thus $XH(\{x\})\neq XH(\{y\})$.

$(3)\Rightarrow (1)$:  Since for all $x,y\in U$ and $x\neq y$, $XH(\{x\})\neq XH(\{y\})$ and the definition of $XH(\{x\})$ and $XH(\{y\})$,
there exists $u\in U$ such that $x\in N(u)$, $y\notin N(u)$ or there exists $v\in U$
such that $x\notin N(v)$, $y\in N(v)$.
Let $A=N(u)$, $B=N(v)$.
Since $N(x)$ and $N(y)$ are open sets, $(U,\mathcal{T}_{XL})$ is a $T_{0}$ space.
\end{proof}

The theorem below establishes some equivalent characterizations of $T_{1}$ space (resp.$T_{2}$ space).

\begin{definition}\cite{Engelking77General}
A topological space $(U,\mathcal{T})$ is a $T_{1}$ space (resp.$T_{2}$ space) if for any two different points $x,y \in U$, there exist open
neighborhoods $A$ of $x$ and $B$ of $y$ such that $y\notin A$ and $x\notin B$ (resp.$A\bigcap B=\emptyset$).
\end{definition}

\begin{proposition}
Let $(U, \mathcal{C})$ be a covering approximation space. The following statements are equivalent.\\
(1) $(U, \mathcal{T}_{XL})$ is a $T_{1}$ space.\\
(2) $\forall x,y\in U$, if $x\neq y$, then $x\notin XH(\{y\})$ and $y\notin XH(\{x\})$.\\
(3) $\forall x\in U$, $XH(\{x\})=\{x\}$.\\
(4) $(U,\mathcal{T}_{XL})$ is a $T_{2}$ space.\\
(5) $\forall x,y \in U$, if $x\neq y$, then $N(x)\bigcap N(y)=\emptyset$.
\end{proposition}

\begin{proof}
$(1)\Rightarrow (2)$: If $(U, \mathcal{T}_{XL})$ is a $T_{1}$ space, then for all $x,y\in U$, $x\neq y$, there exist open neighborhoods
$A$ of $x$ and $B$ of $y$ such that $y\notin A$ and $x\notin B$.
According to (2) of Theorem \ref{thepropertiesoftopologyinducedbythesixtypeofapproximation}, we have $N(x)\subseteq A$ and $N(y)\subseteq B$.
Since $y\notin A$ and $x\notin B$, $y\notin N(x)$ and $x\notin N(y)$.
Thus $x\notin XH(\{y\})$ and $y\notin XH(\{x\})$.

$(2)\Rightarrow (3)$: Combining Proposition \ref{propertiesofthesixtypeofapproximation} with (2), we know for all $x, y\in U$ and $y\neq x$,
$x\in XH(\{x\})$ and $y\notin XH(\{x\})$.
Therefore, $XH(\{x\})=\{x\}$.

$(3)\Rightarrow (1)$: According to (3), we know for all $x,y\in U$ and $x\neq y$, $XH(\{x\})=\{x\}$ and $XH(\{y\})=\{y\}$.
Thus $x\notin XH(\{y\})$ and $y\notin XH(\{x\})$, that is, $y\notin N(x)$ and $x\notin N(y)$.
Since $\{N(x):x\in U\}$ is a base of $\mathcal{T}_{XL}$, there exist open sets $N(x)$ and $N(y)$ such that $y\in N(y),x\notin N(y)$ and
$x\in N(x), y\notin N(x)$.
Therefore, $(U, \mathcal{T}_{XL})$ is a $T_{1}$ space.

$(4)\Rightarrow (5)$: Suppose $(U, \mathcal{T}_{XL})$ is a $T_{2}$ space, then $\forall x,y \in U$ and $x\neq y$, there exist open sets $A,B$ such
that $x\in A,y\in B$ and $A\bigcap B=\emptyset$.
According to (2) of Theorem \ref{thepropertiesoftopologyinducedbythesixtypeofapproximation}, we obtain $N(x)\subseteq A$, $N(y)\subseteq B$.
Since $A\bigcap B=\emptyset$, $N(x)\bigcap N(y)=\emptyset$.

$(5)\Rightarrow (4)$: For all $x,y\in U$ and $x\neq y$, we take $A=N(x)$, $B=N(y)$, thus we can obtain the result.

$(3)\Rightarrow (5)$: If (5) does not hold, then there exist $x,y \in U$ and $x\neq y$ such that $N(x)\bigcap N(y)\neq \emptyset$.
Thus there exists $z\in N(x)\bigcap N(y)$ such that $x,y\in XH(\{z\})$ which contradicts (3).

$(5)\Rightarrow (3)$: From (5), we know that for all $x,y\in U$ and $x\neq y$, $N(x)\bigcap N(y)=\emptyset$.
Thus $y\notin N(x)$ and $x\notin N(y)$, that is, $x\notin XH(\{y\}),y\notin XH(\{x\})$.
Hence, $XH(\{x\})=\{x\}$ for all $x\in U$.
\end{proof}

The following theorem shows some equivalent characterizations of regular space.

\begin{definition}\cite{Engelking77General}
A topological space $(U,\mathcal{T})$ is called regular if for each closed set $A\subseteq U$ and any point $x\notin A$ there are open
sets $W$ and $V$ such that $x\in W,A\subseteq V$ and $W\bigcap V=\emptyset$.
\end{definition}

\begin{proposition}\label{theequivalentoftheregularspace}
Let $(U, \mathcal{C})$ be a covering approximation space. The following statements are equivalent.\\
(1) $(U, \mathcal{T}_{XL})$ is a regular space.\\
(2) For each closed set $A\subseteq U$ and any point $x\notin A$, $N(x)\bigcap N(y)=\emptyset$ for all $y\in A$.\\
(3) For all $x,y \in U$, $N(x)\bigcap N(y)=\emptyset$ or $N(x)=N(y)$.\\
(4) For all $x\in U$, $N(x)$ is a closed set of $U$.
\end{proposition}

\begin{proof}
$(1)\Rightarrow (2)$: Since $(U, \mathcal{T}_{XL})$ is a regular space, for each closed set $A\subseteq U$ and any point $x\notin A$
there are open sets $W$ and $V$ such that $x\in W,A\subseteq V$ and $W\bigcap V=\emptyset$.
$\forall y\in A$, $V$ is an open set which contains $y$.
According to $(2)$ of Theorem \ref{thepropertiesoftopologyinducedbythesixtypeofapproximation}, $N(y)\subseteq V$,
thus $\bigcup_{y\in A}N(y)\subseteq V$.
Similarly, because $W$ is an open set which contains $x$, $N(x)\subseteq W$.
$N(x)\bigcap (\bigcup_{y\in A}N(y))\subseteq W\bigcap V=\emptyset$, thus $\forall y\in A, N(x)\bigcap N(y)= \emptyset$.

$(2)\Rightarrow (3)$: According to (2), if $y\in A$, then $N(x)\bigcap N(y)= \emptyset$.
If $N(x)\bigcap N(y)\neq \emptyset$, then $N(x)=N(y)$; otherwise, there exists $z\in U$ such that $z\in N(x),z\notin N(y)$.
According to the definition of $XH(\{z\})$, we have $x\in XH(\{z\}), y\notin XH(\{z\})$.
Since $XH(\{z\})$ is a closed set, we have $N(x)\bigcap N(y)=\emptyset$ by (2).
That contradicts the assumption that $N(x)\bigcap N(y)\neq \emptyset$.

$(3)\Rightarrow (4)$: It is obvious that $N(x)\subseteq XH(N(x))$, now we need to prove $XH(N(x))\\\subseteq N(x)$. $\forall y\in XH(N(x))$,
$N(x)\bigcap N(y)\neq \emptyset$. According to (3), we know $N(x)=N(y)$, thus $y\in N(y)=N(x)$, that is, $XH(\{N(x)\})\subseteq N(x)$.

$(4)\Rightarrow (1)$: For each closed set $A\subseteq U$ and any point $x\notin A=XH(A)$, we know $N(x)\bigcap A=\emptyset$, that is,
$A\subseteq ((N(x))^{c}$. According to (4), $(N(x))^{c}$ is an open set. Let $W=N(x), V=(N(x))^{c}$.
Thus we have $W\bigcap V=\emptyset$, therefore, $(U, \mathcal{C})$ is a regular space.
\end{proof}

\begin{definition}\cite{Engelking77General}
A topological space $(U, \mathcal{T}_{XL})$ is called normal if for any disjoint closed sets $A$ and $B$ there are open
subsets $W,V\subseteq U$ such that $A\subseteq W, B\subseteq V$ and $W\bigcap V=\emptyset$.
\end{definition}

\begin{proposition}\label{theequivalentofnomalspace}
Let $(U, \mathcal{C})$ be a covering approximation space. The following statements are equivalent.\\
(1) $(U, \mathcal{T}_{XL})$ is a normal space.\\
(2) Assume $A$ and $B$ are any disjoint closed sets, then $\forall x\in A, \forall y\in B$, $N(x)\bigcap N(y)=\emptyset$.\\
(3) Assume $A\subseteq U$ is a closed set, then $\forall y\in A$, $XH(\bigcup_{y\in A}N(y))=\bigcup_{y\in A}N(y)$.
\end{proposition}

\begin{proof}
$(1)\Rightarrow (2)$: If $(U, \mathcal{T}_{XL})$ is a normal space, then for any disjoint closed sets $A,B\subseteq U$ there exist open
sets $W,V\subseteq U$ such that $A\subseteq W, B\subseteq V$ and $W\bigcap V=\emptyset$.
Since $A\subseteq W$, $\forall x\in A\subseteq W$, $N(x)\subseteq W$.
Similarly, $\forall y\in B\subseteq V$, $N(y)\subseteq V$.
Hence, $N(x)\bigcap N(y)=\emptyset$ for $W\bigcap V=\emptyset$.

$(2)\Rightarrow (3)$: It is obvious that $\bigcup_{y\in A}N(y)\subseteq XH(\bigcup_{y\in A}N(y))$, now we need to prove
$XH(\bigcup_{y\in A}N(y))\subseteq \bigcup_{y\in A}N(y)$.
$\forall x\in XH(\bigcup_{y\in A}N(y))$, $N(x)\bigcap (\bigcup_{y\in A}N(y))\\=\bigcup_{y\in A}(N(x)\bigcap N(y))\neq \emptyset$,
that is, there exists $z\in A$ such that $N(x)\bigcap N(z)\neq \emptyset$.
If $\forall y\in A$, $x\notin N(y)\subseteq XH(N(y))$, then $N(x)\bigcap N(y)=\emptyset$ which contradicts
$N(x)\bigcap N(z)\neq \emptyset$.
That implies there exists $y\in A$ such that $x\in N(y)$, that is, $XH(\bigcup_{y\in A}N(y))\subseteq \bigcup_{y\in A}N(y)$.

$(3)\Rightarrow (1)$: Suppose $A, B$ are any closed sets of $U$ satisfies $A\bigcap B=\emptyset$, then $\forall y\in A$, $y\notin B=XH(B)$,
that is, $N(y)\bigcap B=\emptyset$. $(\bigcup_{y\in A}N(y))\bigcap B=\bigcup_{y\in A}(N(y)\bigcap B)=\emptyset$,
thus $B\subseteq (\bigcup_{y\in A}N(y))^{c}$.
Let $W=\bigcup_{y\in A}N(y), V=(\bigcup_{y\in A}N(y))^{c}$. We can obtain the result.
\end{proof}

The following theorem gives us an unexpected result about the relation of normality and regularity of topological space induced by $XL$.

\begin{theorem}\label{relationbetweenregularandnomalspace}
Let $\mathcal{C}$ be a covering. If $(U, \mathcal{T}_{XL})$ is a regular space, then $(U, \mathcal{T}_{XL})$ is a normal space.
\end{theorem}

\begin{proof}
According to (4) of Proposition \ref{theequivalentoftheregularspace}, we know that if $(U, \mathcal{T}_{XL})$ is a regular space
then $\forall x\in U$, $XH(N(x))=N(x)$.
Since $XH(\bigcup_{x\in A}N(x))=\bigcup_{x\in A}XH(N(x))=\bigcup_{x\in A}N(y)$, $(U, \mathcal{T}_{XL})$ is a normal space based on the (3) of
Proposition \ref{theequivalentofnomalspace}.
\end{proof}

The following example illustrates that a normal space is not a regular space.

\begin{example}
Let $U=\{a,b,c\}$, $\mathcal{C}=\{\{a\},\{b\},\{a,b,c\},\}$. $N(a)=\{a\}, N(b)=\{b\}, N(c)=\{a,b,c\}$,
$\mathcal{T}_{XL}=\{\emptyset, \{a\}, \{b\}, \{a,b\}, \{a,b,c\}\}$ and $\mathcal{F}_{XL}=\{\emptyset, \{c\},\{a,c\},$ $\{b,c\},\{a,b,c\}\}$.
There does not exist two disjoint closed sets, then $(U, \mathcal{T}_{XL})$ is a normal space.
However, for closed set $\{b,c\}$ and any point $a$, there dose not exist two disjoint open sets $W,V$ such that
$a\in W, \{b,c\}\subseteq V$, thus $(U,\mathcal{T}_{XL})$ is not a regular space.
\end{example}

The proposition below establishes some topological properties of $(U, \mathcal{T}_{XL})$ under the condition $\{N(x):x\in U\}$ forms a partition.

\begin{proposition}
Let $(U, \mathcal{T}_{XL})$ be a topological space and $\mathcal{C}$ be a covering of $U$. If $\{N(x):x\in U\}$ forms a partition, then \\
(1) $\{N(X):x\in U\}$ is a base of $\mathcal{T}_{XL}$;\\
(2) If $A$ is an open set which contains $x$, then $N(x)\subseteq A$;\\
(3) $\{N(x)\}$ is an open neighborhood base of $x$;\\
(4) $N(x)$ is a compact subset of $(U, \mathcal{T}_{XL})$.\\
(5) $N(x)$ is a connected component which contains $x$.\\
(6) $(U, \mathcal{T}_{XL})$ is a first countable space.\\
(7) $(U, \mathcal{T}_{XL})$ is a locally compact space.\\
(8) $(U, \mathcal{T}_{XL})$ is a locally connected space.\\
(9) $(U, \mathcal{T}_{XL})$ is a regular space.\\
(10) $(U, \mathcal{T}_{XL})$ is a normal space.\\
(11) $(U, \mathcal{T}_{XL})$ is a completely regular space.
\end{proposition}

\begin{proof}
(1)-(4): They can be obtained from (1)-(4) of Theorem \ref{thepropertiesoftopologyinducedbythesixtypeofapproximation}.
(5): Suppose $C_{x}$ is a connected component which contains $x$.
Assume $A$ is a non-empty subset which is both open and closed of $N(x)$, then $A\subseteq N(x)$.
Because $\{N(x):x\in U\}$ forms a partition and (1), $N(x)=N(y)\subseteq A \subseteq N(x)$ for all $y\in A$,
thus $N(x)$ is a connected subset of $U$.
Since $C_{x}$ is a maximum connected subset of $U$, $N(x)\subseteq C_{x}$.
If $C_{x}\neq N(x)$, then $N(x)$ is a non-empty proper subset which is both open and closed of $C_{x}$,
thus $C_{x}$ is disconnected which contradicts that $C_{x}$ is a connected component containing $x$.
(6-8):  According to Theorem \ref{thepropertiesoftopologyinducedbythesixtypeofapproximation}, we obtain these results.
(9-10): According to Theorem \ref{relationbetweenregularandnomalspace} and Theorem \ref{theequivalentoftheregularspace}, we obtain (9) and (10).
(11): As we know, if $(U,\mathcal{T})$ is both regular and normal then it is a completely regular space. By (9) and (10), we prove (11).
\end{proof}

\begin{proposition}
Let $\mathcal{C}$ be a covering of $U$. If $\{N(x):x\in U\}$ forms a partition, then the following statements are equivalent.\\
(1) $\{N(x):x\in U\}$ is countable.\\
(2) $(U, \mathcal{T}_{XL})$ is a second countable space.\\
(3) $(U, \mathcal{T}_{XL})$ is a separable space.\\
(4) $(U, \mathcal{T}_{XL})$ is a $Lindel$\"{o}$f$ space.
\end{proposition}

\begin{proof}
$(1)\Leftrightarrow (2)$: On one hand, $\{N(x):x\in U\}$ is countable and it is also a base of $(U, \mathcal{T}_{XL})$, then
$(U, \mathcal{T}_{XL})$ is a second countable space.
On the other hand, suppose $\{B_{i}:i\in N\}$ is a countable base of $(U, \mathcal{T}_{XL})$, then $\forall N(y)\in \{N(x):x\in U\}$, there
exists $i\in N$ such that $y\in B_{i}\subseteq N(y)$.
Based on (2) of Theorem \ref{thepropertiesoftopologyinducedbythesixtypeofapproximation}, we have $N(y)\subseteq B_{i}\subseteq N(y)$,
that is, $B_{i}=N(y)$.
Then $|\{N(x):x\in U\}| \leq |\{B_{i}:i\in N\}\}|$.
Since $\{B_{i}:i\in N\}$ is countable, then $\{N(x):x\in U\}$ is not more than denumerable set.

$(2)\Rightarrow (3)$ and $(2)\Rightarrow (4)$ are obvious.

$(3)\Rightarrow (1)$: Suppose that $M=\{x_{i}:i\in N\}$ is a countable dense subset of $U$ and $\beta=\{N(x_{i}):i\in N\}$, then $\beta$ is a family of
open sets. Next we need to prove $\beta$ is a base of $\mathcal{T}_{XL}$.
Since $M$ is dense, $XH(M)=\{x:N(x)\bigcap M\neq \emptyset\}=U$.
Thus for all $y\in U$, $N(y)\bigcap M\neq \emptyset$.
Suppose $A$ is an open set and $x\in A$, then $N(x)\bigcap M\neq \emptyset$, thus there exists $x_{i}$ such that $x_{i}\in N(x)$ and $x_{i}\in M$.
Since $\{N(x):x\in U\}$ forms a partition and $x_{i}\in N(x_{i})$, $N(x)=N(x_{i})$.
Based on (2) of Theorem \ref{thepropertiesoftopologyinducedbythesixtypeofapproximation} and $x\in A$, $x\in N(x)=N(x_{i})\subseteq A$.
Hence, $\beta$ is a countable base of $\mathcal{T}_{XL}$,

$(4)\Rightarrow (1)$: $\{N(x):x\in U\}$ is an open covering of $U$ and $\bigcup_{y\in U-\{x\}}N(y)\neq U$, then $\{N(x):x\in U\}$ is not more than
denumerable set since $(U, \mathcal{T}_{XL})$ is a $Lindel $\"{o}$f$ space.
\end{proof}

\section{Topological characterization to the covering lower approximation operator}
\label{Topologicalcharacterizationtothecoveringlowerapproximationoperator}
Section \ref{Thepropertiesoftopologyinducedby} has studied the properties of topology $\mathcal{T}_{XL}$.
In this section, we study the conditions which make the lower approximation operator $CL$ be a interior operator.
Under this condition, it is interesting to find that the topologies induced by $CL$ and $XL$ are the same.

\begin{proposition}\cite{Zhu09RelationshipAmong}\label{9thelowerapproxiamtionandunary}
A covering $\mathcal{C}$ is unary if and only if $CL$ satisfies the following properties:
\begin{center}
 $CL(X\bigcap Y)=CL(X)\bigcap CL(Y)$.
\end{center}
\end{proposition}

\begin{proof}
"$\Leftarrow$": If $\mathcal{C}$ is not unary, there exists $x\in U$ such that $|Md(x)|\geq 2$.
Suppose $K_{1},K_{2}\in \mathcal{C}$ such that $K_{1},K_{2}\in Md(x)$, then
$K_{1}\bigcap K_{2}=CL(K_{1})\bigcap CL(K_{2})=CL(K_{1}\bigcap K_{2})$$=\bigcup \{K\in \mathcal{C}: K\subseteq K_{1}\bigcap K_{2}\}$
based on Proposition \ref{propertiesofthelowerapproximationoperator} and $CL(X\bigcap Y)\\=CL(X)\bigcap CL(Y)$.
Thus there exists $K\in \mathcal{C}$ such that $x\in K\subseteq K_{1}\bigcap K_{2}$ which contradicts $K_{1},K_{2}\in Md(x)$.

"$\Rightarrow$": According to Proposition \ref{propertiesofthelowerapproximationoperator},
we have $CL(X\bigcap Y)\subseteq CL(X)\bigcap CL(Y)$.
Now we need to prove $CL(X)\bigcap CL(Y)\subseteq CL(X\bigcap Y)$.
For all $x\in CL(X)\bigcap CL(Y)$, there exist $K_{1}, K_{2}\in \mathcal{C}$ such that $x\in K_{1}\subseteq X$ and $x\in K_{2}\subseteq Y$,
then $x\in K\bigcap K^{'}\subseteq X\bigcap Y$.
Since $\mathcal{C}$ is unary, we can assume $Md(x)=\{K_{x}\}$.
Thus $K_{x}\subseteq K_{1}$ and $K_{x}\in K_{2}$.
Therefore, there exists $K_{x}\in \mathcal{C}$ such that $x\in K_{x}\subseteq K_{1}\bigcap K_{2}\subseteq X\bigcap Y$.
Hence, $x\in CL(X\bigcap Y)$, that is, $CL(X)\bigcap CL(Y)\subseteq CL(X\bigcap Y)$.
\end{proof}

\begin{proposition}\cite{ZhuWang07OnThree}\label{interioroperatorofthelowerapproximationandunary}
A covering $\mathcal{C}$ is unary if and only if $CL$ is an interior operator.
\end{proposition}

\begin{proof}
It comes from (1L), (3L), (5L) of Proposition \ref{propertiesofthelowerapproximationoperator} and Proposition \ref{9thelowerapproxiamtionandunary}.
\end{proof}

If a operator is the interior operator of a topology, then the fixed point of the operator is an open set of the topology.
As we know, all the open sets forms a topology.
The following theorem construct a topology based on the statements.

\begin{theorem}\label{thetopologyinducedbylowerapproximation}
Covering $\mathcal{C}$ is unary if and only if $\mathcal{T}_{CL}=\{X\subseteq U:CL(X)=X\}$ is a topology on $U$.
\end{theorem}

\begin{proof}
The proof of necessity is similar to that of necessity in Theorem \ref{Thesixthtypeoflowerapproximationsandtopology}, hence we omit it.
Now we consider the proof of sufficiency.
If $\mathcal{T}_{CL}=\{X\subseteq U:CL(X)=X\}$ is a topology, then $\forall X,Y\in \mathcal{T}_{CL}$,
$X\bigcap Y\in \mathcal{T}_{CL}$ according to (O2) of topological atoms.
Hence, $CL(X\bigcap Y)=X\bigcap Y$.
According to Proposition \ref{9thelowerapproxiamtionandunary}, $\mathcal{C}$ is an unary covering.
\end{proof}

The equivalent characterization of an unary covering is often established through covering blocks or minimal description.
However, the two theorems below give an equivalent characterization of an unary covering through topology, respectively.

\begin{theorem}\label{unaryandtopologyinducedbythelowerapproximation}
Covering $\mathcal{C}$ is unary if and only if $\mathcal{C}$ is a base of $\mathcal{T}_{CL}$.
\end{theorem}

\begin{proof}
"$\Rightarrow$": On one hand, we need to prove $\mathcal{C} \subseteq \mathcal{T}_{CL}$. According to Proposition
\ref{propertiesofthelowerapproximationoperator}, it is straightforward.
On the other hand, we need to prove $\mathcal{C}$ is a base of $\mathcal{T}_{CL}$.
For all $X\in \mathcal{T}_{CL}$, $X=CL(X)=\bigcup\{K\in \mathcal{C}:K\subseteq X\}$, thus $\mathcal{C}$ is a base of $\mathcal{T}_{CL}$.

"$\Leftarrow$": If $\mathcal{C}$ is not a unary, there exists $x\in U$ such that $|Md(x)|\geq 2$.
Suppose $K_{1},K_{2}\in \mathcal{C}$ such that $K_{1},K_{2}\in Md(x)$.
Since $\forall K\in \mathcal{C},CL(K)=K$ and $\mathcal{T}_{CL}$ is a topology, $K_{1}\bigcap K_{2}\in \mathcal{T}_{CL}$.
$\mathcal{C}$ is a base of $\mathcal{T}_{CL}$, then $K_{1}\bigcap K_{2}$ can be expressed as the union of some elements of $\mathcal{C}$.
Thus there exists $K\in \mathcal{C}$ such that $x\in K\subseteq K_{1}\bigcap K_{2}$.
That contradicts the fact that $K_{1},K_{2}\in Md(x)$.
\end{proof}

The above theorem characterizes the unary covering from the viewpoint of base, and the following theorem gives an equivalent characterization from
the topological space.

\begin{lemma}\label{unaryandthesixandfirstlowerapproxiamtion}
$\mathcal{C}$ is an unary covering of $U$ if and only if $CL=XL$.
\end{lemma}

\begin{theorem}\label{relationbetweenCLandXL}
$\mathcal{C}$ is unary if and only if $\mathcal{T}_{CL}=\mathcal{T}_{XL}$.
\end{theorem}

\begin{proof}
It comes from Theorem \ref{Thesixthtypeoflowerapproximationsandtopology},
\ref{thetopologyinducedbylowerapproximation} and Lemma \ref{unaryandthesixandfirstlowerapproxiamtion}.
\end{proof}

As we know, $XL$ and $XH$ is dual approximation operators and $CL=XL$ under the condition $\mathcal{C}$ is an unary covering.
Thus we can obtain the following equivalent characterization of the unary covering.

\begin{proposition}
$\mathcal{C}$ is an unary covering if and only if $XL$ is the interior operator and $XH$ is the closure operator of $\mathcal{T}_{CL}$,
respectively.
\end{proposition}

From Theorem \ref{relationbetweenCLandXL}, we know that the topologies induced by $XL$ and $CL$ are the same.
So we omit the discussion of the topological properties of $\mathcal{T}_{CL}$.

\section{Topological characterization to the first type of covering approximation operators}
\label{Theconditionofbeingatopologyandthetopologyinducedby}

In this section, we establish the topological characterization to the first type of covering upper approximation operator based on finite 
universe. 
First, we study the conditions which make the first type of approximation be the closure operator of a topology.
It is interesting to find that the topology induced by the first type of upper covering approximation operator is equal to that of the fifth type
of covering upper approximation operator under these conditions. 

\begin{lemma}\label{Thefirsttypeofupperapproximationandunary}
$\mathcal{C}$ be an unary covering of $U$ if and only if $FH$ is a closure operator.
\end{lemma}

The above lemma establishes the necessary and sufficient condition for $FH$ to be a closure operator.
However, it is not the closure operator of $\mathcal{T}_{CL}(or \mathcal{T}_{XL})$.

\begin{example}
Let $\mathcal{C}=\{\{1,5\},\{1,2,5\},\{3,4\}\}$ be a covering of $U=\{1,2,3,4,5\}$.
$Md(1)=Md(5)=\{1,5\}$, $Md(2)=\{1,2,5\}$, $Md(3)=Md(4)=\{3,4\}$, that is, $\mathcal{C}$ is a unary covering.
Let $X=\{2,3,4\}$.
$CL(X)=CL(\{2,3,4\})=\{3,4\}$ and $-CL(-X)=-CL(\{1,5\})=\{2,3,4\}$.
$FH(\{2,3,4\})=CL(\{2,3,4\})\bigcup(\bigcup\{\bigcup Md\\(x):x\in X-CL(X)\})=U\neq -CL(-X)$.
Thus $CL$ and $FH$ are not dual, that is, $FH$ is not the closure operator of $\mathcal{T}_{CL}(or \mathcal{T}_{XL})$ for $CL$
is the interior operator of
$\mathcal{T}_{CL}(or \mathcal{T}_{XL})$.
\end{example}

The following result is only a combination of Theorem \ref{thetopologyinducedbylowerapproximation}
and Lemma \ref{Thefirsttypeofupperapproximationandunary}.

\begin{theorem}
$FH$ is a closure operator if and only if $\mathcal{T}_{CL}$ is a topology on $U$.
\end{theorem}

Combining Theorem \ref{unaryandtopologyinducedbythelowerapproximation} with Lemma \ref{Thefirsttypeofupperapproximationandunary},
we have a characterization for $FH$ to be a closure operator through topological base.

\begin{theorem}
$FH$ is a closure operator if and only if $C$ is a base of $\mathcal{T}_{CL}$.
\end{theorem}

There are other topological characterizations for $FH$ to be a closure operator.

\begin{corollary}
$FH$ is a closure operator if and only if $\mathcal{T}_{CL}=\mathcal{T}_{XL}$.
\end{corollary}

\begin{corollary}
If $FH$ is a closure operator if and only if $\{N(x):x\in U\}$ is a base of $\mathcal{T}_{CL}$.
\end{corollary}

\begin{corollary}
$FH$ is a closure operator if and only if $XL$ and $XH$ are respectively the interior and closure operator of $\mathcal{T}_{CL}$.
\end{corollary}

The results below establish the topological structure induced by $FH$.

\begin{lemma}\label{unaryandMd(x)andN(x)}
$\mathcal{C}$ is an unary covering of $U$ if and only if $Md(x)=\{N(x)\}$.
\end{lemma}

\begin{proof}
On one hand, $\mathcal{C}$ is an unary covering of $U$, then $|Md(x)|=1$, thus $N(x)=\bigcap Md(x)=Md(x)$.
On the other hand, according to the definition of $Md(x)$ and $Md(x)=\{N(x)\}$, then $N(x)\in \mathcal{C}$.
Thus $|Md(x)|=1$, that is, $\mathcal{C}$ is an unary covering.
\end{proof}

\begin{proposition}\label{relationbetweenthefirsttypeofapproximationandthefifthtypeofapproxiamtion}
If $\mathcal{C}$ is an unary covering of $U$, then $FH=IH$.
\end{proposition}

From Proposition \ref{propertiesofthefifthtypeofapproximation},
we know that $IH$ is a closure operator, thus we obtain a topology $\mathcal{T}_{IH}=\{-X:IH(X)=X\}$.
When $\mathcal{C}$ is an unary covering, the operator $FH$ is a closure operator.
We denote the topology whose closure operator is $FH$ as $\mathcal{T}_{FH}$.
Combining with Proposition \ref{relationbetweenthefirsttypeofapproximationandthefifthtypeofapproxiamtion},
we obtain the following result.

\begin{theorem}
If $\mathcal{C}$ is an unary covering of $U$, then $\mathcal{T}_{FH}=\mathcal{T}_{IH}$.
\end{theorem}

%

\section{Topological characterization to the second type of covering approximation operators}
\label{Topologicalcharacterizationtothesecondtypeofcoveringapproximationoperators}
Pomykala first studied the second type of covering rough set model~\cite{Pomykala87Approximation}.
Zhu and Wang studied the axiomatization of this type of approximation and the relationship between it and
the closure operator in \cite{ZhuWang07OnThree}.
In this section, we establish other topological equivalent characterizations of this type of covering upper approximation operator to be a
closure operator.

\begin{proposition}\label{I(x)andSH(X)}
Let $\mathcal{C}$ be a covering. $SH(SH(X))=SH(X)$ if and only if $\{I(x):x\in U\}$ induced by $\mathcal{C}$ forms a partition.
\end{proposition}

\begin{proof}
"$\Leftarrow$": According to (2) and (5) of Proposition 9, we have $SH(X)\subseteq SH(SH(X))$.
Now we prove $SH(SH(X))\subseteq SH(X)$.
For all $x\in SH(SH(X))$, there exists $y\in SH(X)$ such that $x\in I(y)$.
Since $y\in SH(X)$, there exists $z\in X$ such that $y\in I(z)$.
According to the definition of $I(y)$, we know $y\in I(y)$, thus $I(z)\bigcap I(y)\neq \emptyset$.
For $\{I(x):x\in E\}$ forms a partition, $I(z)=I(y)$.
Since $x\in I(y)$, $x\in I(z)$, that is, $x\in SH(X)$, thus $SH(SH(X))\subseteq SH(X)$.

"$\Rightarrow"$: In order to prove $\{I(x):x\in E\}$ forms a partition, we need to prove that for all $x, y \in E$,
if $I(x)\bigcap I(y)\neq \emptyset$, then $I(x)=I(y)$.
If $I(x)\bigcap I(y)\neq \emptyset$, then there exists $z\in I(x)\bigcap I(y)$.
For $SH(SH(\{x\}))=\bigcup\{I(u):u\in I(x)\}$ and $z\in I(x)$, then $I(z)\subseteq SH(SH(\{x\}))=SH(\{x\})=I(x)$.
Based on the definition of $I(z)$ and $z\in I(x)$, we have $x\in I(z)$, thus $I(x)\subseteq SH(SH(\{z\}))=SH(\{z\})=I(z)$.
Hence, $I(x)=I(z)$.
Similarly, we obtain $I(y)=I(z)$, thus $I(x)=I(z)=I(y)$.
\end{proof}

\begin{proposition}\cite{GeBaiYun12Topological}\label{theconditionforSHtobeaclosureoperator}
$SH$ is a closure operator if and only if $\{I(x):x\in U\}$ forms a partition.
\end{proposition}

\begin{proof}
It comes from (2), (3), (4) of Proposition \ref{propertiesofthesecondapproximationoperator} and Proposition \ref{I(x)andSH(X)}.
\end{proof}

As we know, a closure operator $cl$ can induce a topology $\mathcal{T}=\{-X:cl(X)=X\}$.
The theorem below establishes the structure of topology induced by $SH$.

\begin{theorem}
Let $\mathcal{C}$ be a covering of $U$. If $\{I(x):x\in U\}$ induced by $\mathcal{C}$ forms a partition, then $\mathcal{T}_{SH}=\{-X:SH(X)=X\}$
is a preudo-discrete topology on $U$. Moreover, $\{I(x):x\in U\}$ is a base of $\mathcal{T}_{SH}$.
\end{theorem}

\begin{proof}
According to Proposition \ref{theconditionforSHtobeaclosureoperator}, we know $SH$ is a closure operator.
Thus $\mathcal{T}_{SH}=\{-X:SH(X)=X\}$ is a topology on $U$.

Next, we prove $\mathcal{T}_{SH}$ is a preudo-discrete topology, that is, $X$ is closed if and only $X$ is open.
On one hand, if $X$ is closed set, then $SH(X)=X=\bigcup_{x\in X}I(x)$.
According to Proposition \ref{propertiesofthesecondapproximationoperator}, we have $-X\subseteq SH(-X)$.
For all $y\in SH(-X)=\bigcup_{z\in -X}I(z)$, there exists $z\in -X$ such that $y\in I(z)$, thus $I(y)=I(z)$ for $\{I(x):x\in U\}$ forms a partition.
If $I(z)\bigcap X=I(z)\bigcap (\bigcup_{x\in X}I(x))=\bigcup_{x\in X}(I(z)\bigcap I(x))\neq \emptyset$, then there exists $x\in X$ such that
$I(z)\bigcap I(x)\neq \emptyset$.
Since $\{I(x):x\in U\}$ forms a partition, $I(x)=I(z)$.
Thus $z\in I(z)=I(x)\subseteq \bigcup_{x\in X}I(x)=X$ which contradicts $z\in -X$.
Therefore, $I(z)\bigcap X=\emptyset$, that is, $y\in I(y)=I(z)\subseteq -X$, which implies $SH(-X)=-X$, that is, $-X$ is closed, then $X$ is open.
On the other hand, if $X$ is open, then $-X$ is closed.
Thus $-X$ is open, therefore $X$ is open.

Finally, we prove $\{I(x):x\in U\}$ is a base of $\mathcal{T}_{SH}$.
Since $\{I(x):x\in U\}$ forms a partition, $SH(I(x))=\bigcup_{y\in I(x)}I(y)=I(x)$ for all $x\in U$.
Thus $I(x)$ is closed set.
Since $\mathcal{T}_{SH}$ is a preudo-discrete topology, $I(x)$ is a open set.
Hence $\{I(x):x\in U\}\subseteq \mathcal{T}_{SH}$.
Since $\mathcal{T}_{SH}$ a preudo-discrete topology, for all $X\in \mathcal{T}_{SH}$, $X=SH(X)=\bigcup_{x\in X}I(x)$.
Therefore, $\{I(x):x\in U\}$ is a base of $\mathcal{T}_{SH}$.

\end{proof}
\section{Relationships among these three topological spaces}
\label{Relationshipsamongthesethreetopologicalspaces}
In the previous sections, we establish topological characterizations to five types of covering approximation operators.
It is interesting to find that $CL$ as a closure operator is equal to $XL$ and $FH$ as a closure operator is equal to $IH$.
In other words, the topologies induced by $CL$ and $XL$ are the same and the topologies induced by $FH$ and $IH$ are the same.
This section studies the relationships among three topological spaces, namely, $\mathcal{T}_{CL}$, $\mathcal{T}_{FH}$ and $\mathcal{T}_{SH}$.

\begin{proposition}\cite{ZhuWang07OnThree}
If $FH(-X)=-CL(X)$, then $\forall K_{1},\cdots, K_{m}\in \mathcal{C}$, $-(K_{1}\bigcup \cdots \\\bigcup K_{m})$ is a union of finite elements
in $\mathcal{C}$.
\end{proposition}

\begin{proposition}
If $\mathcal{T}_{FH}=\mathcal{T}_{CL}$, $\forall K_{1},\cdots, K_{m}\in \mathcal{C}$, $-(K_{1}\bigcup \cdots \bigcup K_{m})$ is a
union of finite elements in $\mathcal{C}$.
\end{proposition}

The following one theorem presents the relationship between $\mathcal{T}_{SH}$ and $\mathcal{T}_{CL}$.

\begin{theorem}\label{relationbetweentopologiesinducedbySHandCL}
Let $\mathcal{C}$ be a covering of $U$. $\mathcal{T}_{SH}=\mathcal{T}_{CL}$ if and only if $C$ is a partition.
\end{theorem}

\begin{proof}
"$\Leftarrow$": Since $\mathcal{C}$ is a partition, then $CL=R_{\ast}$ and $SH=R^{\ast}$, thus $\mathcal{T}_{SH}=\mathcal{T}_{CL}$ for
$\mathcal{T}_{R_{\ast}}=\mathcal{T}_{R^{\ast}}$.

"$\Rightarrow$":
If $\mathcal{C}$ is not a partition, then there exist $K_{1}, K_{2}\in \mathcal{C}$ such that $K_{1}\bigcap K_{2}\neq \emptyset$.
Thus $K_{1}-K_{2}\neq \emptyset$ or $K_{2}-K_{1}\neq \emptyset$ holds.
We might as well suppose $K_{1}-K_{2} \neq \emptyset$, then $K_{2}\subset K_{1}\bigcup K_{2}\subseteq I{\{x\}}$ for all $x\in K_{1}\bigcap K_{2}$.
Hence, $I(\{x\})\nsubseteq K_{2}$.
According to the definition of $SL$, we know $x\notin SL(K_{2})$, that is, $SL(K_{2})\neq K_{2}$.
Thus $K_{2} \notin \mathcal{T}_{SL}=\mathcal{T}_{SH}$, i.e. $\mathcal{T}_{SH} \neq \mathcal{T}_{CL}$ which
contradicts the assumption that $\mathcal{T}_{SH}=\mathcal{T}_{CL}$.
Therefore, $C$ is a partition.
\end{proof}

The following theorem shows the relationship between topologies induced by $FH$ and $SH$, respectively.

\begin{proposition}\label{relationofSH(X)andFH(X)}
If $\mathcal{C}$ is unary, then $\forall X\subseteq U,FH(SH(X))=SH(X)$.
\end{proposition}

\begin{proof}
According to Proposition \ref{relationbetweenthefirsttypeofapproximationandthefifthtypeofapproxiamtion} and
\ref{propertiesofthefifthtypeofapproximation}, we know $FH=IH$ and $SH(X) \subseteq IH(SH(X))$ for all $X\subseteq U$.
Now we need to prove $IH(SH(X))\subseteq SH(X)$.
$\forall x\in IH(SH(X))$, there exists $y\in SH(X)$ such that $x\in N(y)$.
Since $y\in SH(X)$, there exists $z\in X$ such that $y\in I(z)$, that is, there exists $K\in \mathcal{C}$ such that $y,z\in K$,
thus $x\in N(y)\subseteq K$.
So $x,y,z\in K$ which implies $x\in I(z)$, that is, $x\in SH(X)$.
Therefore, $FH(SH(X))=SH(X)$.
\end{proof}

\begin{theorem}
Let $\mathcal{C}$ be a covering of $U$. If $FH$ and $SH$ induced by $\mathcal{C}$ are closure operators, then $\mathcal{T}_{FH}\subseteq \mathcal{T}_{SH}$.
\end{theorem}

\begin{proof}
$FH$ induced by $\mathcal{C}$ is a closure operator, then $\mathcal{C}$ is unary.
Based on Proposition \ref{relationofSH(X)andFH(X)}, we have $\forall X\in\mathcal{T}_{SH}$, $X=SH(X)=IH(SH(X))=FH(SH(X))=FH(X)$.
That implies $X\in \mathcal{F}_{FH}$, thus $\mathcal{F}_{SH}\subseteq \mathcal{F}_{FH}$.
Therefore, we obtain $\mathcal{T}_{FH}\subseteq \mathcal{T}_{SH}$.
\end{proof}

When covering degenerates into a partition, those topologies induced by five types of covering approximation operators are the same.

\begin{theorem}
If $\mathcal{C}$ is a partition if and only if
$\mathcal{T}_{SH}=\mathcal{T}_{CL}=\mathcal{T}_{FH}=\mathcal{T}_{XL}=\mathcal{T}_{XH}=\mathcal{T}_{IH}$.
\end{theorem}

\begin{proof}
According to Definition \ref{definitionofloweroperator} and \ref{relationbetweentopologiesinducedbySHandCL},
we obtain the result.
\end{proof}
\section{Conclusion}
\label{Conslusion}
This paper has presented the topological characterizations to five types of covering approximation operators, namely, the lower and the first,
the second, the fifth and the sixth type of approximation operator.
We found that the topologies induced by the lower approximation and by the sixth type of covering approximation operator are the same, and
the topology induced by the first type of covering approximation operator and the one induced by the fifth type of approximation operator.
Many problems still remain to be solved.
Hence our future works are to present the topological characterizations to covering-based rough sets.
\section{Acknowledgments}
This work is supported in part by the National Natural Science Foundation of China under Grant No. 61170128,
the Natural Science Foundation of Fujian Province, China, under Grant Nos. 2011J01374 and 2012J01294,
and the Science and Technology Key Project of Fujian Province, China, under Grant No. 2012H0043.



\end{document}